\documentclass[sigconf,screen,authorversion,nonacm]{acmart}

\AtBeginDocument{%
  \providecommand\BibTeX{{%
    \normalfont B\kern-0.5em{\scshape i\kern-0.25em b}\kern-0.8em\TeX}}}

\setcopyright{acmcopyright}
\copyrightyear{2023}
\acmYear{2023}
\acmDOI{XXXXXXX.XXXXXXX}

\acmConference[CIKM '23]{32nd ACM International Conference on Information and Knowledge Management}{October 21 - October 25, 2023}{Birmingham , UK}
\acmPrice{15.00}
\acmISBN{978-1-4503-XXXX-X/18/06}

\usepackage{pifont}
\newcommand{\cmark}{\ding{51}}%
\newcommand{\xmark}{\ding{55}}%
\usepackage{graphicx}
\usepackage{amsthm}
\usepackage{amsmath}
\usepackage[noabbrev]{cleveref}
\usepackage{enumitem}
\usepackage{lineno}
\newtheorem{theorem}{Theorem}
\theoremstyle{definition}
\newtheorem{definition}{Definition}
\usepackage{makecell}
\usepackage{subcaption}
\usepackage{makecell}
\usepackage{fancyhdr}
\begin{document}

\title{Simple Rule Injection for ComplEx Embeddings}

\author{Haodi Ma}
\email{ma.haodi@ufl.edu}
\affiliation{%
  \institution{University of Florida}
  \city{Gainesville}
  \state{Florida}
  \country{USA}
}

\author{Anthony Colas}
\email{acolas1@ufl.edu}
\affiliation{%
  \institution{University of Florida}
  \city{Gainesville}
  \state{Florida}
  \country{USA}
}

\author{Yuejie Wang}
\email{yw4989@nyu.edu}
\affiliation{%
  \institution{New York University}
  \city{New York}
  \state{New York}
  \country{USA}
}

\author{Ali	Sadeghian}
\email{ali.sadeghian2050@gmail.com}
\affiliation{%
  \institution{University of Florida}
  \city{Gainesville}
  \state{Florida}
  \country{USA}
}

\author{Daisy Zhe Wang}
\email{daisyw@cise.ufl.edu}
\affiliation{%
  \institution{University of Florida}
  \city{Gainesville}
  \state{Florida}
  \country{USA}
}


\begin{abstract}
  Recent works in neural knowledge graph inference attempt to combine logic rules with knowledge graph embeddings to benefit from prior knowledge. However, they usually cannot avoid rule grounding and injecting a diverse set of rules has still not been thoroughly explored. In this work, we propose InjEx, a mechanism to inject multiple types of rules through simple constraints, which capture definite Horn rules. To start, we theoretically prove that InjEx can inject such rules. Next, to demonstrate that InjEx infuses interpretable prior knowledge into the embedding space, we evaluate InjEx on both the knowledge graph completion (KGC) and few-shot knowledge graph completion (FKGC) settings. Our experimental results reveal that InjEx outperforms both baseline KGC models as well as specialized few-shot models while maintaining its scalability and efficiency.
\end{abstract}



\maketitle

\section{Introduction}
Knowledge graphs (KGs) are a collection of triples, where each triple represents a relation \textit{r} between the head entity \textit{h} and tail entity \textit{t}. Examples of real-world KGs include Freebase~\cite{bollacker2008freebase}, Yago~\cite{suchanek2008yago} and NELL~\cite{carlson2010toward}. These KGs contain millions of facts and are fundamental basis for applications like question answering, recommender systems and natural language processing. 



The immense amount of information stored in today's large scale KGs can be used to infer new knowledge directly from the KG. Currently there the two prominent approaches to achieve this are representation learning and pattern mining over KGs. \textbf{KG embedding (KGE) models} aim to represent entities and relations as low-dimensional vectors, such that the semantic meaning is captured. This approach has been well studied in the past decade~\cite{bordes2013translating,yang2015embedding, wang2014knowledge, lin2015learning,ji2015knowledge}. However these approaches still lack of interpretability and face challenges with unseen entities.

Other KGE works focus on leveraging more complex triple scoring models~\cite{wang2022dirie} or using meta learning~\cite{chen2019meta, xiong2018one, lv2019adapting, wang2021mulde}, while possible prior knowledge are relatively ignored. Few works that study injecting rules into embeddings do so for a very limited types of rules~\cite{ding2018improving, hu2016harnessing}. We show that one can use simple constraints on the embeddings and objective function to inject a variety of rules types. 

On the other hand, \textbf{KG rule mining} techniques extract information in the form of human understandable logical rules. For example, AMIE~\cite{galarraga2013amie}, DRUM~\cite{sadeghian2019drum}, and AnyBURL~\cite{ijcai2019-435} discover and mine meaningful symbolic rules from the background KGs. Specifically,AMIE~\cite{galarraga2013amie} states composition rules are the most common and important ones among all the rules. Based on those, rule-based works like SERFAN~\cite{ott2021safran} focus on predicting missing links between entities with certain type of rules. A major bottleneck of the currently used approaches is the relatively lower predictive performance compared to KG representation learning methods. 

Recently there has been attempts at combining embedding based and rule based methods to achieve both higher performance and better interpretability, for example ComplEx-NNE~\cite{ding2018improving} and UniKER~\cite{cheng-etal-2021-UniKER}. In addition, rule injection methods provide a natural way of including prior knowledge into representation learning techniques. 
Nevertheless, previous attempts use probabilistic model to approximate the exact logical inference~\citep{qu2019probabilistic, qu2020rnnlogic, zhang2019can} requires grounding of rules which is intractable in large-scale real-world knowledge graphs. Other works treat logical rules as additional constraints into KGE~\citep{guo-etal-2016-jointly, demeester2016lifted} usually deal with a certain type of rules or do not explore the effect between different types of rules.


In this work, we propose InjEx, a novel method of rule injection that improves the reasoning performance and provides the ability of soft injection of prior-knowledge via multiple different types of rules without grounding (see Table~\ref{model-patterns}).
Following the idea of ComplEx-NNE~\cite{ding2018improving}, we propose that with proper constraints on entity and relation embeddings, we are able to handle definite Horn rules with ComplEx as the base model. First, we impose a non-negative bounded constraint on both entity and relations embeddings. Second, we add a simple yet novel regularization to the KG embedding objective that enforces the rules' constraints. As we will explain in more details in Section~\ref{sec:rule-injection}, the former guarantees that the base model is able to inject multiple types of patterns and the latter encodes the connections between relations, which helps the model learn more predictive embeddings. 

To demonstrate the effectiveness InjEx in incorporating prior knowledge, we also evaluate InjEx on multiple widely used benchmarks for link prediction. We show that InjEx achieves, or is on par with, state-of-the-art models on all the benchmarks and multiple evaluation settings.

Our contributions can be summarized as follows:
\begin{itemize}

\item We propose a novel model, InjEx, that allows integration of prior knowledge in the form of definite Horn rules into ComplEx KG embeddings. InjEx only requires minimal modifications of the underlying embedding method which avoids rule grounding and enables soft rule injection.

\item Unlike prior methods~\citep{ding2018improving}, InjEx is not limited to a specific rule structure. We show that InjEx is able to capture definite Horn rules with any body length and empirically outperforms state-of-art models on knowledge graph completion (KGC) task. InjEx improves around 5\% in Hits@10 on both FB15k-237 and YAGO3-10 comparing with its base model ComplEx. We also show the effect of the combination or separately injecting multi-length definite Horn rules on KGC tasks.

\item We show that, on few-shot link prediction task, InjEx still successfully captures the injected prior knowledge and is able to achieve competitive performance against other baselines on such task. 
\end{itemize}

\begin{table}[!h]
  \caption{Capabilities of different methods.}
  \label{model-patterns}
  \centering
  \begin{tabular}{lccc}
    \toprule
            & \makecell[c]{Multi-length\\definite Horn rules}   & \makecell[c]{Soft rule\\ injection} & \makecell[c]{Avoid\\Grounding}\\
    \midrule
    \ Demeester et al.\cite{demeester2016lifted} & \xmark  & \xmark & \cmark  \\
    \ ComplEx-NNE\cite{ding2018improving}   & \xmark  & \cmark & \cmark  \\
    \midrule
    \ KALE\cite{guo-etal-2016-jointly}  & \cmark  & \xmark & \xmark \\
    \ RUGE\cite{guo2018knowledge}  & \cmark  & \cmark & \xmark \\
    \midrule
    \ RNNLogic\cite{qu2020rnnlogic}  & \cmark  & \cmark &\xmark \\
    \ pLogicNet\cite{qu2019probabilistic}  & \cmark  & \cmark &\xmark \\
    \midrule
    \ UniKER\cite{cheng-etal-2021-UniKER}  & \cmark  & \cmark &\xmark \\
    \ InjEx  & \cmark  & \cmark &\cmark \\
    \bottomrule
  \end{tabular}
\end{table}

\section{Related work}
In this section, we give an overview of embedding models for KGC and few-shot link prediction, rule mining systems and constraint/rule assisted link prediction models.

\subsection{Knowledge Graph Embedding (KGE) Models}
KG embedding models can be generally classified into translation and bilinear models. The representative of translation models is TransE~\citep{bordes2013translating}, which models the relations between entities as the difference between their embeddings. This method is effective in inferencing composition, anti-symmetry and inversion patterns, but can't handle the 1-to-N, N-to-1 and N-N relations. RotatE~\citep{sun2018rotate} models relations as rotations in complex space so that symmetric relations can be captured, but is as limited as TransE otherwise. 
Other models such as BoxE~\citep{abboud2020boxe}, HAKE~\citep{zhang2020learning} and DiriE~\citep{wang2022dirie} are able to express multiple types of relation patterns with complex KG embeddings. MulDE~\citep{wang2021mulde} proposes to transfer knowledge from multiple teacher KGE models to perform better on KGC tasks. ComplEx~\citep{trouillon2016complex}, as a representative of bilinear models, introduces a diagonal matrix with complex numbers to capture anti-symmetry. 

\subsection{Few-shot Knowledge Graph Completion}
Few-shot KGC refers to the scenario where only a limited number of instances are provided for certain relations or entities, which means the model needs to leverage knowledge about other relations or entities to predict the missing link. In this work we are mainly concerned about few-shot learning for new relations\cite{chen2019meta, bose2019meta, jambor2021exploring}. One of the early works in this direction~\cite{xiong2018one}, leverages a neighbor encoder to learn entity embeddings and a matching component to find similar reference triples to the query triple. 

\subsection{Rule-based Models}
Previous works such as AMIE~\cite{galarraga2013amie}, DRUM~\cite{sadeghian2019drum}, mine logical rules to predict novel links in KGs.SAFRAN~\cite{ott2021safran} and AnyBURL~\cite{ijcai2019-435} share a similar method to predict links with logical rules. Similar to InjEx, ComplE-NNE\cite{ding2018improving} approximately applies entailment patterns as constraints on relation representations to improve the KG embeddings. However, it can only inject entailment rules. Another way is to augment knowledge graphs with grounding of rules which is less efficient for large scale KGs. Representatives like RUGE~\cite{guo2018knowledge} and KALE~\cite{guo-etal-2016-jointly} treats logical rules as additional regularization by computing satisfaction score to sample ground rules.  IterE~\cite{zhang2019iteratively} [98] proposes an iterative training strategy with three components of embedding learning, axiom induction, and axiom injection, targeting at sparse entity reasoning.
UniKER~\cite{cheng-etal-2021-UniKER} augments triplets from relation path rules to improve embedding quality. But to avoid data noise, it uses only small number of relation path rules and requires multiple passes of augmenting data and model training. 

\subsection{Graph Neural Network Models}
The Graph Neural Network (GNN) has gained wide attention on KGC tasks in recent years~\cite{wang2021mixed, zhang2022rethinking, yu2021knowledge}. With the high expressiveness of GNNs, these methods have shown promising performance. However, SOTA GNN-based models do not show great advantages compared with KGE models while introducing additional computational complexity~\cite{zhang2022rethinking}. For example, 
RED-GNN~\cite{zhang2022knowledge} achieves competitive performance on KGC benchmarks, but the leverage of the Bellman-Ford algorithm which needs to propagate through the whole knowledge graph, which restrict their application on large graphs. Several methods including pLogicNet~\citep{qu2019probabilistic} proposes to using Markov Logic Network (MLN) to compute variational distribution over possible hidden triples for logic reasoning. RNNLogic~\cite{qu2020rnnlogic} learns logic rules for knowledge graph reasoning with EM-based algorithm in reinforcement learning.

\begin{figure*}[t]
\centering
\includegraphics[width=0.75\textwidth]{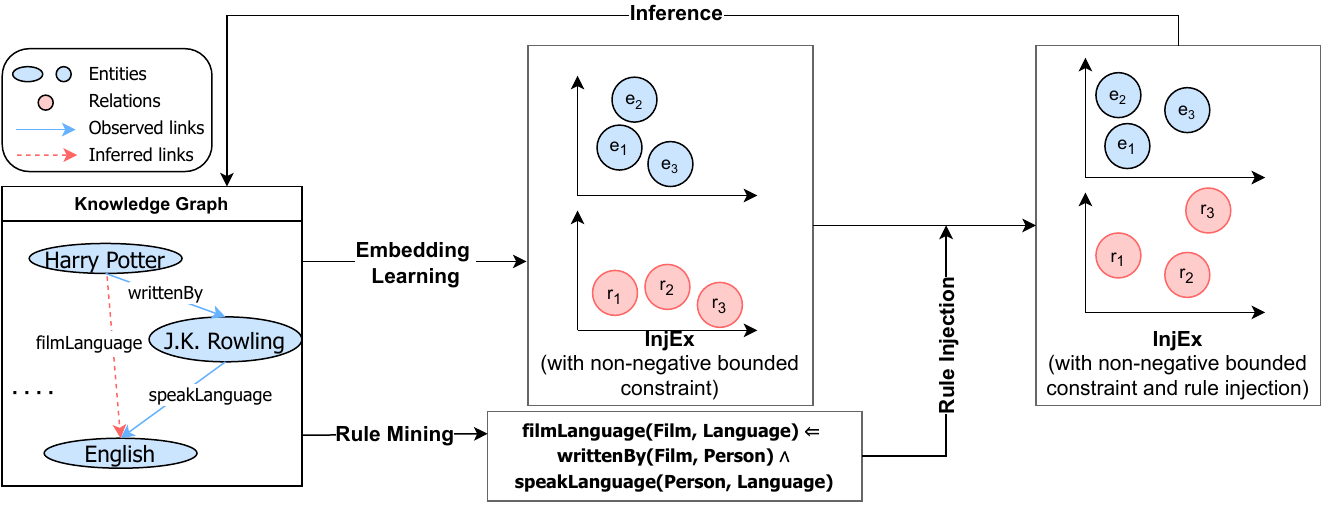}
\caption{Overview of the InjEx model. Note, we illustrate injecting a composition rule.}
\label{fig:framework}
\end{figure*}

\section{PRELIMINARIES} 
\subsection{Knowledge Graphs (KGs)}
Let $\mathcal{E}$ and $\mathcal{R}$ denote the set of entities and relations, a knowledge
graph $\mathcal{G} = \{(e_i, r_k, e_j)\} \subset \mathcal{E} \times \mathcal{R}\times \mathcal{E}$ is a collection of factual triples, where $e_i$ and $r_k$ are the $i$-th entity and $k$-th relation, respectively. We usually refer $e_i$ and $e_j$ as the head and tail entity. 
A knowledge graph can also be represented as $\mathcal{X} \in \{0, 1\}^{|\mathcal{E}| \times |\mathcal{R}| \times |\mathcal{E}|}$, which is called the adjacancy tensor of $\mathcal{G}$. The $(i, j, k)$ entry $\mathcal{X}_{i, k, j} = 1$ when triple $(e_i, r_k, e_j)$ is true, otherwise $\mathcal{X}_{i, k, j} = 0$. 
\subsection{Knowledge Graph Completion (KGC) and ComplEx}
\subsubsection{\textbf{Knowledge Graph Completion}}
The objective of KGC is to predict valid but unobserved triples in $\mathcal{G}$. Formally, given a head entity $e_i$ (tail entity $e_j$) with a relation $r_k$, models are expected to find the tail entity $e_j$ (head entity $e_i$) to form the most plausible triple $(e_i, r_k, e_j)$ in $\mathcal{G}$. KGC models usually define a scoring function $f: \mathcal{E} \times \mathcal{R}\times \mathcal{E} \rightarrow \mathbb{R}$ to assign a score $s(e_i, r_k, e_j)$ to each triple $(e_i, r_k, e_j) \in \mathcal{E} \times \mathcal{R}\times \mathcal{E}$ which indicates the plausibility of the triple. 

KGE models usually associate each entity $e_i$ and relation $r_j$ with vector representations $\textbf{e}_i$, $\textbf{r}_j$ in the embedding space. Then they define a scoring function to model the interactions among entities and relations. We review ComplEx, which is used as our base model.
\subsubsection{\textbf{ComplEx}} \label{sec:basemodel}
ComplEx~\cite{trouillon2016complex} models entities $\emph{e} \in \mathcal{E}$ and relations $\emph{r} \in \mathcal{R}$  as complex-valued vectors $\textbf{e},\textbf{r} \in \mathbb{C}^d$ where d is the embedding space dimension. For each triple $(e_i, r_k, r_j) \in \mathcal{E} \times \mathcal{R}\times \mathcal{E} $, the scoring function is defined as:

\begin{align}
    \label{eq:complex-score}
    \phi(e_i, r_k, e_j) &:= Re(\langle \textbf{e}_i, \textbf{r}_k, \overline{\textbf{e}}_j \rangle) \nonumber\\
    &:= Re(\sum_l [\textbf{e}_i]_l[\textbf{r}_k]_l[\overline{\textbf{e}}_j]_l)
\end{align}
where $\textbf{e}_i, \textbf{r}_k, \textbf{e}_j \in \mathbb{C}^d$ are the vector representations associated with $e_i, r_k, r_j$ and $\overline{\textbf{e}}_j$ is the conjugate of $\textbf{e}_j$. Triples with higher $\phi(\cdot, \cdot, \cdot)$ scores are more likely to be true. We use $\langle \textbf{a}, \textbf{b}, \textbf{c} \rangle$ or $\textbf{a} \times \textbf{b}$ for element-wise vector multiplication. For a complex number $x \in \mathbb{C}$, $Re(\cdot), Im(\cdot)$ means taking the real, imaginary components of $x$.

\subsection{Definite Horn Rules} \label{sec:horn rule}
Horn rules, as a popular subset of first-order logic rules, can be automatically extracted from recent rule mining systems, such as AMIE~\cite{galarraga2013amie} and AnyBurl~\cite{ijcai2019-435}. Length $k$ Horn rules are usually written in the form of implication as below:
\begin{align*}
\forall x, y, r_1(x, z_1) \land r_2(z_1, z_2) \land ... r_k(z_{k-1}, y) \Rightarrow r(x, y) 
\end{align*}
where $r(x, y)$ is called the head of the rule and $r_1(x, z_1) \land r_2(z_1, z_2) \land ... r_k(z_{k-1}, y)$ is the body of the rule. In KGs, a ground Horn rule is then represented as:
\begin{align*}
    r_1(e_i, e_{p_1}) \land r_2(e_{p_1}, e_{p_2}) \land ... r_k(e_{p_{k-1}}, e_j) \Rightarrow r(e_i, e_j) 
\end{align*}
For convenience, we define length-1 and length-2 definite Horn rules as:
\begin{definition}[Hierarchy rule]
\label{def:hierarchy rule}
A \textbf{hierarchy} rule holds between relation $r_1$ and $r_2$ if $\forall e_i, e_j$
\begin{center}
    $r_1(e_i, e_j) \Rightarrow r_2(e_i, e_j)$
\end{center}

\end{definition}
\begin{definition}[Composition rule]
\label{def:composition rule}
A \textbf{composition} rule holds between relation $r_1$, $r_2$ and $r_3$ if $\forall e_i, e_j, e_k$
\begin{center}
    $r_1(e_i, e_j) \land r_2(e_j, e_k) \Rightarrow r_3(e_i, e_k)$
\end{center}
\end{definition}

\section{METHEDOLOGY} \label{sec:model}
Here, we first introduce how the non-negative constraint over both entity and relation embeddings help make the model fully expressive (\ref{sec:NNE}). Then we discuss how we integrate Horn rules into the base model as a rule based regularization (\ref{sec:rule-injection}). Figure~\ref{fig:framework} illustrates an overview of the InjEx model for composition rules.

\subsection{Non-negative Bounded Constraint on Entity and Relation Representations} \label{sec:NNE}
We modify ComplEx further by requiring both entity and relation representations to be non-negative and bounded. More formally we require entity and relation embeddings to satisfy:
\begin{align}
    \label{eq:boundedNNE}
    & 0 \leq Re(\textbf{e}), Im(\textbf{e}) \leq 1, \forall e \in \mathcal{E}\nonumber\\
    & 0 \leq Re(\textbf{r}), Im(\textbf{r}), 0 \leq |\textbf{r}| \leq R, \forall r \in \mathcal{R}
\end{align}
where $R \in \mathbb{R}$ is a selected upper bound for the norm of relation representations. 

As pointed out by~\cite{murphy2012learning}, the positive elements of entities or relations usually contain enough information. Thus, intuitionally we don't expect these constraints to significantly effect the performance of ComplEx. While they provide additional benefits such as: i) guarantee that the original scoring function of~\cite{trouillon2016complex} is bounded. ii) as we demonstrate in theorem~\ref{thm:composition}, with the constraints above ComplEx can now also infer composition patterns.
\subsection{Rule Injection} \label{sec:rule-injection}
In this section, we discuss how Horn rules are integrated into ComplEx as a rule based regularization. We start with the injection of composition rules, extending the discussion to length-k Horn rules. 

We treat each dimension of an entity as a separate attribute. Recall the equation~\ref{eq:complex-score}, for the simplicity of notations, we decompose
\begin{align}
\label{eq:attr_score}
    \phi(e_i,r_k,e_j) &=: \Sigma_l(\varphi_l(e_i,r_k,e_j)) \nonumber\\
    \varphi_l(e_i,r_k,e_j) &:= Re([e_i]_l[r_k]_l[\bar{e}_j]_l) \\
    &= (|[e_i]_l||[r_k]_l||[e_j]_l|\cos{(\theta_{[r_k]_l} + \theta_{[e_i]_l} - \theta_{[e_j]_l})}) \nonumber
\end{align}
s.t. $\varphi_l(e_i,r_k,e_j)$ denotes the score of dimension $l$ in relation $r_k$.

\vspace{2mm}
\noindent\textbf{Composition Rule Injection. }
We now introduce how we inject composition rules. 
As in definition~\ref{def:composition rule}, composition rules express one type of relation between three relations. For example, $\texttt{BornIn}(x, y) \land \texttt{CountryOfOrigin}(y, z) \rightarrow \texttt{Nationality}(x, z)$ indicates The country of the city which a person is born in should usually be the nationality of that person. 
We denote such composition rules with confidence level $\lambda$ as 
\begin{align*}
r_1(x, y) \land r_2(y, z) \xrightarrow{\lambda} r_3(x, z)
\end{align*}

As defined in equation~\ref{eq:attr_score}, given the range constraints, we have $\varphi_l(x,r,y) / (2R) \in [0,1]$, so we can treat $\varphi_{l} / (2R)$ as the probability that the relation holds for the triple's attribute $l$.

Under such assumption, we now propose we can roughly model a strict composition rule $r_1(e_i, e_j) \land r_2(e_j, e_k) \xrightarrow{\infty} r_3(e_i, e_k)$ as follows. 
\begin{theorem}
\label{thm:composition}
A strict composition rule $r_1(e_i, e_j) \land r_2(e_j, e_k) \xrightarrow{\infty} r_3(e_i, e_k)$ holds if the following condition is satisfied. 
\begin{align} 
\label{eq:composition rule score}
    \varphi_l(e_i, r_1, e_j) * \varphi_l(e_j, r_2, e_k) / (2R) \leq \varphi_l(e_i, r_3, e_k) \nonumber \\
    \forall e_i, e_j, e_k \in \mathcal{E}, \forall 1 \leq l \leq d
\end{align}
where d is the dimension of $e$'s and $r$'s, R is the range of all relation embeddings as defined in Section~\ref{sec:NNE}. (See proof in Appendix \ref{prf: constraint proof})
\end{theorem}

Note that with the non-negative bounded constraint still, a sufficient condition for Equation~\ref{eq:composition rule score} to hold, is to further impose
\begin{align}
\label{eq:composition rules}
    & Re(\textbf{r}_1 \times \textbf{r}_2) / R \leq Re(\textbf{r}_3), \nonumber \\
    & Im(\textbf{r}_1 \times \textbf{r}_2) / R = Im(\textbf{r}_3)
\end{align}
where $\textbf{r}_1, \textbf{r}_2, \textbf{r}_3$ are the vectorized representation of $\textit{r}_1$, $\textit{r}_2$ and $\textit{r}_3$. We further introduce the level of confidence and slackness in Equation~\ref{eq:composition rules} to model approximate composition rules, which yields
\begin{align}
\label{eq:composition rules-confidence term}
    & \lambda (Re(\textbf{r}_1 \times \textbf{r}_2) / R^2 - Re(\textbf{r}_3) / R) \leq \alpha, \nonumber \\
    & \lambda (Im(\textbf{r}_1 \times \textbf{r}_2) / R^2 - Im(\textbf{r}_3) / R)^2 \leq \beta
\end{align}

\vspace{2mm}
\noindent\textbf{Definite Horn Rule Injection. }
We now introduce how we inject Horn rules based on previous observations. 
As defined in Sec\ref{sec:horn rule}, length-k Horn rules represent one type of relationship between $k + 1$ relations. For example, $\texttt{FilmWrittenBy}(x, y) \rightarrow \texttt{FilmDirectedBy}(x, y)$ indicates that a person who is the writer of a movie should probably also direct that movie; $ActorOfFilm(x, z) \land nationality(z, y) \rightarrow FilmReleaseRegion(x, y)$ shows a more complex relation between the actor, her nationality and the region where the film is released. We further denote Horn rule with confidence level $\lambda$ as
\begin{align*}
    r_1(x, z_1) \land r_2(z_1, z_2) \land ... \land &r_k(z_{k-1}, y) \xrightarrow{\lambda} r(x, y) \nonumber \\
    or \ r_1 \land r_2 \land ... \land &r_k \xrightarrow{\lambda} r
\end{align*}

With the same assumption as in composition rule injection, we can further model a strict Horn rule $r_1(x, z_1) \land r_2(z_1, z_2) \land ... \land r_k(z_{k-1}, y) \xrightarrow{\infty} r(x, y)$ as follows:
\begin{theorem}
\label{thm:Horn}
A strict Horn rule $$r_1(x, z_1) \land r_2(z_1, z_2) \land ... \land r_k(z_{k-1}, y) \xrightarrow{\infty} r(x, y)$$ holds if the following condition is satisfied: 
\begin{align} 
\label{eq:Horn rule score}
    \Pi_{i=1}^k(\varphi_{i, l}(z_{i-1}, r_i, z_i)/R) \leq (\varphi_l(x, r, y)/R), \nonumber
    \\ \forall e_i, e_j, e_{p_1}, ..., e_{p_{k-1}} \in \mathcal{E}, \forall 1 \leq l \leq d,
\end{align}
where $x = z_0$, $y = z_k$, R is still the range of all relation embeddings as defined in Section~\ref{sec:NNE}. 
\end{theorem}

\begin{proof}

If we have triples $(z_{i-1}, r_i, z_i)$ for $1 \leq i \leq k$,
\begin{align*}
    \varphi_{i,l} &:= Re(z_{i-1,l}*r_{i,l}*\bar{z}_{i,l}) \\
    &= |r_{i,l}||z_{i-1,l}||z_{i,l}|\cos{(\theta_{r_{i,l}} + \theta_{z_{i-1,l}} - \theta_{z_{i,l}})} \\
    \phi_{i} &:= \Sigma_l(\varphi_{i,l})
\end{align*}

If we have $rule(r_1 \land r_2 \land ... \land r_k \longrightarrow r)$ for triples $(x, r, y)$ s.t. $x = z_0$ and $y = z_k$. Let $\Pi$ be the element-wise multiplication, let

\begin{align}
\label{def_hat_r}
    \hat{r} &:= R * \Pi_{i=1}^k (r_i / R) \in (0,R] \\
    \hat{\varphi}_l &:= Re(z_{0,l}*\hat{r}_l*z_{k,l}) \nonumber\\
\label{def_hat_phi}
    &= |\hat{r}_l||z_{0,l}||z_{k,l}|\cos{(\theta_{\hat{r}_l} + \theta_{z_{0,l}} - \theta_{z_{k,l}})}
\end{align}

By definition,
\begin{align*}
    \varphi_l &:= Re(x_l*r_l*\bar{y}_l) \\
    \phi &:= \Sigma_l(\varphi_l)
\end{align*}
let $r$ satisfies that (with confidence level $\lambda$):
\begin{align}
\label{def_r}
    &\lambda(Re(\hat{r}) / R - Re(r) / R) \leq \alpha \nonumber\\
    &\lambda(Im(\hat{r}) / R - Im(r) / R) \leq \beta
\end{align}

For each relation $(x', r', y')$ and $1 \leq l \leq n$, where n is the dimension of the vectors, we have $\varphi_{r'_l} / (2R) \in [0,1]$, so we can treat $\varphi_{r'_l} / (2R)$ as the probability that the relation holds for the triple's attribute $l$. We imply $rule(r_1 \lor r_2 \lor ... \lor r_k \longrightarrow r)$ by enforcing that on each dimension $l$, $\Pi_{i=1}^k(\varphi_{i,l} / (2R)) \leq \varphi_l / (2R)$. 

Now we prove that our definitions above can imply the rule. For each $1 \leq i \leq k$, relation holds for triple $(z_{i-1}, r_i, z_i)$, so $\varphi_i$ is maximized on each dimension. So we have for each $1 \leq l \leq n$, $$(\theta_{r_{i,l}} + \theta_{z_{i-1,l}} - \theta_{z_{i,l}}) = 0$$
Therefore,
\begin{align*}
    \Pi_{i=1}^k(\varphi_{i,l} / (2R)) &= \Pi_{i=1}^k((|r_{i,l}| / R)|z_{i-1,l}||z_{i,l}|/2) \\
    &= \Pi_{i=1}^k(|r_{i,l}| / R) * |z_{0,l}| * \Pi_{i=1}^{k-1}(|z_{i,l}|^2/2) * |z_{k,l}|/2
\end{align*}
Notice that by the definition of (\ref{def_hat_r}) and (\ref{def_hat_phi}),
\begin{align*}
    &\theta_{\hat{r}_l} = \Sigma_{i=1}^k\theta_{r_{i,l}} \\
    &|\hat{r}_l| / R = \Pi_{i=1}^k(|r_{i,l}| / R) \\
    &\theta_{\hat{r}_l} + \theta_{z_{0,l}} - \theta_{z_{k,l}} = \Sigma_{i=1}^k(\theta_{r_{i,l}} + \theta_{z_{i-1,l}} - \theta_{z_{i,l}}) = 0
\end{align*}
\begin{align*}
    \Longrightarrow \hat{\varphi}_l / R &= (|\hat{r}_l| / R)|z_{0,l}||z_{k,l}|\cos{(\theta_{\hat{r}_l} + \theta_{z_{0,l}} - \theta_{z_{k,l}})} \\
    &= \Pi_{i=1}^k(|r_{i,l}| / R)|z_{0,l}||z_{k,l}|
\end{align*}
Given that $|z_{i,l}| \leq 1$,
$$|z_{0,l}| * \Pi_{i=1}^{k-1}(|z_{i,l}|^2) * |z_{k,l}| \leq |z_{0,l}||z_{k,l}|$$
So we have $$\Pi_{i=1}^k(\varphi_{i,l} / R) \leq \hat{\varphi}_l / R$$ According to \cite{ding2018improving}, our choice of $r$ in (\ref{def_r}) guarantees that $$\Pi_{i=1}^k(\varphi_{i,l} / R) \leq \hat{\varphi}_l / R \leq \varphi_l / R$$
\end{proof}
With the non-negative bounded constraint in \ref{sec:NNE}, a sufficient condition for Equation~\ref{eq:Horn rule score} to hold, is to further impose
\begin{align}
\label{eq:Horn rules}
    & Re(\textbf{r}_1 \times \textbf{r}_2 \times ... \times \textbf{r}_k) / R^k \leq Re(\textbf{r}) / R, \nonumber \\
    & Im(\textbf{r}_1 \times \textbf{r}_2 \times ... \times \textbf{r}_k) / R^k = Im(\textbf{r}) / R
\end{align}
where $\textbf{r}_i$ is the vectorized representation of $\textit{r}_i$. We further introduce the level of confidence and slackness in Equation~\ref{eq:Horn rules} to model approximate composition rules, which yields
\begin{align}
\label{eq:Horn rules-confidence term}
    & \lambda (Re(\textbf{r}_1 \times \textbf{r}_2 \times ... \times \textbf{r}_k) / R^k - Re(\textbf{r}) / R) \leq \alpha, \nonumber \\
    & \lambda (Im(\textbf{r}_1 \times \textbf{r}_2 \times ... \times \textbf{r}_k) / R^k - Re(\textbf{r}) / R)^2 \leq \beta
\end{align}

\subsection{Optimization Objective} \label{sec:objective}
We combine together the basic embedding model of ComplEx, the non-negative bounded constraint on both entity and relation representations and the approximate injection of Horn rules as InjEx. The overall model is presented as follows:
\begin{align}
\label{eq:optimization objective-1}
    \min_{\Theta, \{\alpha_1, \beta_1, \alpha_2, \beta_2 \}}& \sum_{\mathcal{D+\cup D^-}}{log(1+exp(-y_{ijk}\phi(e_i, r_k, e_j)))}  \nonumber \\
    & + \mu\sum_{\mathcal{T}}{\textbf{1}^\top (\alpha + \beta)} + \eta||\Theta||_3^3,   \nonumber \\
    s.t.\ 
    & \lambda (Re(\textbf{r}_1 \times ... \times \textbf{r}_k) / R^k - Re(\textbf{r}) / R) \leq \alpha, \nonumber \\
    & \lambda (Im(\textbf{r}_1 \times ... \times \textbf{r}_k) / R^k - Re(\textbf{r}) / R)^2 \leq \beta,   \nonumber \\
    & \alpha, \beta \geq 0, \nonumber \\
    & \forall r_1 \land ... \land r_k \xrightarrow{\lambda} r \in \mathcal{T}   \nonumber \\
    & 0 \leq Re(\textbf{e}), Im(\textbf{e}) \leq 1, \forall e \in \mathcal{E}\nonumber\\
    & 0 \leq Re(\textbf{r}), Im(\textbf{r}), 0 \leq |\textbf{r}| \leq R, \forall r \in \mathcal{R}
\end{align}
Here, $\Theta := \{\textbf{e}: e \in \mathbb{E}\} \cup \{\textbf{r}: r \in \mathbb{R}\}$ is the set of all entity and relation representations; $\mathcal{D}+$ and ${\mathcal{D}^-}$ are the sets of positive and negative training triples respectively; $\mathcal{T}$ is the rule set of Horn rules. The first term of the objective function is a typical logistic loss as in ComplEx. The second term is the sum of slack variables, used to inject Horn rules with penalty coefficient $\mu \geq 0$. We encourage the slackness to be as small as possible to better satisfy the rules. The last term is N3 regularization~\cite{lacroix2018canonical} to avoid overfitting. 

To solve the optimization problem, we convert slack variable for Horn rules into penalty terms and add to the original objective function and rewrite Equation~\ref{eq:optimization objective-1} as:
\begin{align}
\label{eq:optimization objective-2}
    \min_{\Theta, \{\alpha_1, \beta_1, \alpha_2, \beta_2 \}}& \sum_{\mathcal{D+\cup D^-}}{log(1+exp(-y_{ijk}\phi(e_i, r_k, e_j)))}  \nonumber \\
    & + \mu \sum_{\mathcal{T}}{\lambda \textbf{1}^\top [Re(\textbf{r}_1 \times ... \times \textbf{r}_k) / R^k - Re(\textbf{r}) / R]_+} \nonumber \\
    & + \mu \sum_{\mathcal{T}}{\lambda \textbf{1}^\top (Im(\textbf{r}_1 \times ... \times \textbf{r}_k) / R^k - Re(\textbf{r}) / R))^2}    \nonumber \\
    & + \eta||\Theta||_3^3,     \nonumber \\
    s.t.\ 
    & 0 \leq Re(\textbf{e}), Im(\textbf{e}) \leq 1, \forall e \in \mathcal{E}   \nonumber\\
    & 0 \leq Re(\textbf{r}), Im(\textbf{r}), 0 \leq |\textbf{r}| \leq R, \forall r \in \mathcal{R}
\end{align}
where $[\textbf{v}]_+ = max(\textbf{0}, \textbf{v})$, where $max(\cdot, \cdot)$ is an element-wise operation and $\lambda$ is the confidence of a corresponding rule. We use AdaGrad~\cite{duchi2011adaptive} as our optimizer. After each gradient descent step, we project entity and relation representations into $[0, 1]^d$ and $[0, R]^d$, respectively. When training, we set the norm boundary for relations \textit{R} = 1.

While forming a better structured embedding space, Injecting multiple types of rules only have small impact on model complexity. InjEx has the same $O(nd + md)$ space complexity as ComplEx, where $n$ is the number of entities, $m$ is the number of relations and d is the dimension of complex-valued vector for entity and relation representations. The time complexity of InjEx is $O((\overline{t} + p + \overline{n})d)$ where $\overline{t}$ is the average number of triples in each batch, $\overline{n}$ is the average number of entities in each batch and $p$ is the total number of rules we inject. In a real scenario, the number of rules is usually much smaller than the number of triples, i.e. $\overline{t} >> p$ and $\overline{n} \leq \overline{t}$. Thus, the time complexity of InjEx is also $O(\overline{t}d)$, which is the time complexity of ComplEx.

\section{Experiments} \label{sec:Experiments}
In this section we evaluate InjEx on two tasks, KGC and FKGC and report state-of-the-art results. The result on KGC task show that InjEx improve the KG embedding while the result on FKGC task confirm that InjEx uses the prior knowledge in prediction. 

In our experiment, We limit the maximum length of rules to 2 for the efficiency and quality of mining valid rules. Thus, rules are classified into two types according to their length: (1) The set of length-1 rules are hierarchy rules. (2) The set of rules with length 2 are composition rules. For ablation study, we report the performance of only injecting composition rules (InjEx-C) or hierarchy rules (InjEx-H).

\subsection{Knowledge Graph Completion} \label{sec:KGC exps}
\subsubsection{\textbf{Datasets}} We evaluate the effectiveness of our InjEx for link prediction on three real-world KGC benchmarks: FB15k-237, WN18RR~\cite{bordes2013translating} and YAGO3-10~\cite{mahdisoltani2014yago3}. FB15k-237 excludes inverse relations from FB15k and includes 14541 entities, 237 relations and 272,155 training triples. WN18RR is constructed from WordNet~\cite{miller1995wordnet} with 40,943 entities, 11 relations and 93,013 triples. YAGO3-10~\cite{mahdisoltani2014yago3} is a subset of YAGO, containing 123,182 entities, 37 relations and 1,079,040 training triples with almost all common relation patterns.

\subsubsection{\textbf{Baselines}}We compare InjEx with KGE models: TransE~\cite{bordes2013translating}, RotatE~\cite{sun2018rotate}, ComplEx-N3~\cite{lacroix2018canonical}, BoxE~\cite{abboud2020boxe}, and  HAKE~\cite{zhang2020learning}; 
Rule-based or rule injection models: DRUM~\cite{sadeghian2019drum}, SAFRAN~\cite{ott2021safran}, RUGE~\cite{guo2018knowledge}, KALE~\cite{guo-etal-2016-jointly},  ComplEx-NNE~\cite{ding2018improving} and UniKER~\cite{cheng-etal-2021-UniKER}; GNN/GCN model: pLogicNet~\cite{qu2019probabilistic}, RNNLogic~\cite{qu2020rnnlogic}, and RED-GNN~\cite{zhang2022knowledge}. 

\subsubsection{\textbf{Rule Set}} We generate our rule set via AnyBURL to all the aforementioned benchmarks. We further extract length-1 and length-2 Horn rules with confidence higher than 0.5. As such, we extract 2552/10/172 hierarchies and 149/10/65 compositions for FB15k-237/WN18RR/YAGO3-10, respectively.

\subsubsection{\textbf{Experiment Setup}}We report two common evaluation metrics for both tasks: mean reciprocal rank (MRR), and top-10 Hit Ratio (Hit@10). For each triple in the test set, we replace its head entity $e_o$ with every entity $e_i \in \mathcal{E}$ to create candidate triples in the link prediction task. We evaluate InjEx in a filtered setting as mentioned in~\cite{bordes2013translating}, where all corrupted triples that already exist in either training, validation or test set are removed. All candidate triples are ranked based on their prediction scores. Higher MRR or Hit@k indicates better performance. 

We ran the same grid of hyper-parameters for all models on the FB15K-237, WN18RR, NELL-One and FB15k-237-Zero datasets. Our grid includes a learning rate $\gamma \in \{0.1, 0.2, 0.5\}$, two batch-sizes: 25 and 100, and regularization coefficients $\eta \in \{0, 0.001, 0.005, 0.01, \\0.05, 0.1, 0.5\}$. On YAGO3-10, we limited InjEx to embedding sizes $\textit{d} = 1000$ and used batch-size 1000, as ~\cite{abboud2020boxe} only reports their result with $\textit{d} \leq 1000$, keeping the rest of the grid fixed. All experiments were conducted on a shared cluster of NVIDIA A100 GPUs. 

We train InjEx for 100 epochs, validating every 20 epochs. For all models, we report the best published results. The best results are highlighted in bold. 

For ComplEx-NNE, we experiment on FB15k-237 with the same rule set as InjEx-C, which are not reported in the original paper. Other baselines results are taken from~\cite{sun2018rotate},~\cite{abboud2020boxe}~\cite{lacroix2018canonical}, ~\cite{ott2021safran}, and  ~\cite{cheng-etal-2021-UniKER}.

\begin{table*}
  \caption{KGC results (MRR, Hits@10) for InjEx and competing approaches on FB15k-237 and YAGO3-10. We re-run ComplEx-NNE on FB15k-237 and report the result. Other approach results are as best published.}
  \label{KGC-result1}
  \centering
  \begin{tabular}{lllllll}
    \toprule
    &\multicolumn{2}{c}{FB15k-237}  &\multicolumn{2}{c}{WN18RR}   &\multicolumn{2}{c}{YAGO3-10}   \\
    \cmidrule(r){2-3} \cmidrule(r){4-5} \cmidrule(r){6-7}
    Model       & MRR   & Hit@10    & MRR   & Hit@10    & MRR   & Hit@10    \\
    \midrule
    TransE~\cite{bordes2013translating}     & .332 & .531       & .226  & .501  & .501  & .673  \\
    RotatE~\cite{sun2018rotate}             & .338  & .533      & .476  & .571  & .498  & .670  \\
    ComplEx-N3~\cite{lacroix2018canonical}  & .370  & .560      & .48  & .57    & .580  & .710  \\
    BoxE~\cite{abboud2020boxe}              & .337  & .538      & .451  & .541  & .567  & .699  \\
    HAKE~\cite{zhang2020learning}           & .346  & .542      & .497  & .582  & .545  & .694  \\
    \cmidrule(r){1-7}
    DRUM~\cite{sadeghian2019drum}           & .343  & .516      & .486  & .586  & -  & -  \\
    SAFRAN~\cite{ott2021safran}             & .389  & .537      & .502  & .578  & .564  & .693  \\
    \cmidrule(r){1-7}
    RUGE~\cite{guo2018knowledge}            & .191  & .376      & .280  & .327  & -  & -  \\
    KALE~\cite{guo-etal-2016-jointly}       & .230  & .424      & .172  & .353  & -  & -  \\
    ComplEx-NNE~\cite{ding2018improving}$^*$ & .373  & .555     & .481  & .580  & .590  & .721  \\
    UniKER-RotatE~\cite{cheng-etal-2021-UniKER} & \textbf{.539}  & .612  & .492  & .580 & -  & -    \\
    \cmidrule(r){1-7}
    pLogicNet~\cite{qu2019probabilistic}    & .332  & .524      & .441  & .537  & -  & -    \\
    RNNLogic~\cite{qu2020rnnlogic}          & .349  & .533      & .513  & .597  & -  & -    \\
    RED-GNN~\cite{zhang2022knowledge}       & .374  & .558      & \textbf{.533}  & \textbf{.624} & -  & -  \\
    \cmidrule(r){1-7}
    InjEx-H                                 & .390  & .560      & 0.482  & 0.580    & .632 & .754   \\
    InjEx-C                                 & .408  & .598      & 0.481  & 0.581    & .610 & .751   \\
    InjEx                                   & .420  & \textbf{.615} & 0.483  & 0.588    & \textbf{.660} & \textbf{.761}   \\
    \bottomrule
  \end{tabular}
\end{table*}

\subsubsection{\textbf{Main Results}} Table~\ref{KGC-result1} presents our results on FB15k-237, WN18RR and YAGO3-10. For FB15k-237, by injecting definite Horn rules, InjEx outperforms the state-of-art on Hits@10. Note, InjEx surpasses all embedding-standalone and rule-mining models, showing the effectiveness of rule injection. Strong performance on FB15k-237, which originally contains several composition patterns, suggests that InjEx can perform well with compositions even they are not explicitly inferred. 

For WN18RR, InjEx achieves competitive performance against all embedding-standalone and rule-mining models. It worth noting that WN18RR only has 11 relations which limits the quality and amount of inference patterns. With such constraint, InjEx still outperforms all the basic KGE models and most of representative methods on combining logical ruls and embedding models.

For YAGO3-10, InjEx outperforms all state-of-the-art models, with significant improvement over RotatE, BoxE, HAKE and SAFRAN. The result is encouraging since YAGO3-10 is more complex than FB15k-237 and WN18RR, because of the various combinations of different types of inference patterns.  Since~\cite{abboud2020boxe} and~\cite{song2021rot} both have good pattern inference, the strong performance of InjEx on FB15k-237 and YAGO3-10 indicates that InjEx can more efficiently capture different types of patterns. 

Overall, InjEx is competetive on all benchmarks and is state of the art on YAGO3-10. Especially, it perform better than (or at least equally well as) ComplEx-N3, ComplEx-NNE in both matrics on all three benchmarks. It shows that with the soft constraint and penalty-term injection, InjEx does improve the quality of KG embedding. Hence, its is a effective and strong model that leverage prior knowledge for KGC on large real-world KGs. 

\subsubsection{\textbf{Ablation Studies}} Table~\ref{KGC-result1} also presents the performance of InjEx-C and InjEx-H on FB15k-237, WN18RR and YAGO3-10. The performance is under the same parameter setting of the reported InjEx, which means we did not fine-tune for one type of relation. 

For FB15k-237 and WN18RR, by injecting either hierarchy rules or composition rules, InjEx-H and InjEx-C both achieves competitve performance. For YAGO3-10, both InjEx-H and InjEx-C outperform all state-of-the-art model. Further, the results of InjEx-H and InjEx-C against ComplEx-N3 show that InjEx is able to separately leverage two types of patterns to improve the quality of KG embeddings.  

The result shows that different KGs have different distributions on the type of rules and InjEx is able to effectively used both composition and hierarchy rules, either separately or together, to improve the quality of KG embeddings.

\subsection{Few-shot Knowledge Graph Completion} \label{sec:FKGC exps}
\subsubsection{\textbf{Dataset}} We evalute InjEx on two FKGC benchmarks:NELL-One~\cite{xiong2018one} and FB15k-237-Zero. NELL-One~\cite{xiong2018one} is generated from NELL~\cite{carlson2010toward} by removing inverse relations.It contains 68,545 entities, 358 relations and 181,109 triples. 51/5/11 relations are selected as task relations for training, validation and testing. Each task relation has only one triple in the corresponding set. FB15k-Zero is a dataset constructed from FB15k-237 following the same setting of NELL-One. We randomly select 8 relations as the task relations in the test set. We extract all triples with task relations from FB15k-237 as our test set. We randomly add 0, 1, 3 and 5 triples for each task relation into the training set and remove those from the test set to evaluate the effectiveness of InjEx on leveraging prior knowledge in different few-shot settings.

\subsubsection{\textbf{Baselines}} We compare InjEx with FKGC models GMatching~\cite{xiong2018one} and MetaR~\cite{chen2019meta}, KGC models TransE~\cite{bordes2013translating}, ComplEx~\cite{trouillon2016complex} and Dismult~\cite{yang2015embedding} on NELL-One and with ComplEx-N3~\cite{lacroix2018canonical} on FB15k-237-Zero. 

\begin{table*}
  \caption{FKGC experiment results on NELL-One with 1-shot setting}
  \label{FKGC-result1}
  \centering
  \begin{tabular}{cccc|cc|ccc}
    \toprule
    & TransE    & DisMult & ComplEx & GMatching & MetaR   & InjEx-H & InjEx-C & InjEx\\
    \midrule
    MRR     & .105  & .165  & .179  & .185  & \textbf{.250} & .240  & .203  & .245  \\
    \midrule
    Hit@10  & .226  & .285  & .239  & .279  & .261  & .304 & .267 &  \textbf{.320}  \\
    \bottomrule
\end{tabular}
\end{table*}

\subsubsection{\textbf{Rule Set}} To get Horn rules for FB15K-237-Zero, we select only those from FB15K-237 rule set, with task relations in FB15k-237-Zero as its head relation, obtaining 209 hierarchy rules and 11 composition rules. For NELL-One, we use AnyBURL in the same fashion as in the KGC task to extract Horn rules. Since a small number of composition rules with high confidence compared to hierarchy rules are found for NELL-One, we include all the composition rules but only hierarchy rules with confidence $\geq 0.8$, resulting in 3023 hierarchies and 38 compositions for NELL-One. 

\subsubsection{\textbf{Experiment Setup}} We evaluate InjEx on 1-shot setting for NELL-One, providing 1 triple for each target relation in the training set. To leverage the supporting triples in ComplEx-N3 and InjEx, we use the supporting triples in the training phase. On FB15k-237-Zero, we evaluate from 0-shot to 5-shot setting on both ComplEx-N3, InjEx-H, InjEx-C. We report the MRR and Hits@10 for both tasks. 

We ran the same grid of hyper-parameters for all models on the NELL-One and FB15k-237-Zero datasets. Our grid includes a learning rate $\gamma \in \{0.1, 0.2, 0.5\}$, two batch-sizes: 25 and 1000, and regularization coefficients $\eta \in \{0, 0.001, 0.005, 0.01, 0.05, 0.1, 0.5\}$. Other training settings are the same as in the KGC tasks. 

\begin{table*}
  \caption{FKGC results (MRR, Hits@10) on FB15k-237-Zero with k-shot settings}
  \label{FKGC-result2}
  \begin{tabular}{lllllllll}
    \toprule
    &\multicolumn{2}{c}{0-shot} &\multicolumn{2}{c}{1-shot} &\multicolumn{2}{c}{3-shot} &\multicolumn{2}{c}{5-shot} \\
    \cmidrule(r){2-3} \cmidrule(r){4-5} \cmidrule(r){6-7} \cmidrule(r){8-9}
    FB15k-237-Zero      & MRR   & Hit@10    & MRR   & Hit@10    & MRR   & Hit@10    & MRR   & Hit@10    \\
    \midrule
    ComplEx-N3          & .0011 & .0016     & .127  & .180      & .134  & .181      & .162  & .243      \\
    \midrule
    InjEx-H      & .07   & .197      & .08  & .204       & .114  & .231       & .151  & .302       \\
    InjEx-C    & \textbf{.123} & \textbf{.224}  & \textbf{.146}  & \textbf{.247}    & \textbf{.158}  & \textbf{.252}   & \textbf{.207}  & \textbf{.325}  \\
    \bottomrule
  \end{tabular}
\end{table*}

\subsubsection{\textbf{Results}}
Table~\ref{FKGC-result1} summarizes our results on NELL-One. InjEx outperforms all traditional KGC approaches, highlighting that prior rules improve the KG embedding of few-shot relations. InjEx-H (with hierarchy rules) also achieves significant improvement compared to GMatching and is competitive with MetaR. Although GMatching considers the neighborhood information in their model, the lack of learning other patterns that exist in the original KG limits their performance. In contrast, InjEx is capable of leveraging information provided by various types of patterns.

Table\ref{FKGC-result2} further illustrates InjEx effectively leveraging prior knowledge on the FKGC task. On FB15k-237-Zero, with 0 to 5 supporting triples provided for each task relation, InjEx-Composition consistently outperforms ComplEx-N3. Injecting hierarchy rules improves performance on Hits@10 for all settings but only for 0-shot on MRR. One possible reason is that Freebase does not have a clear ontology, effecting the quality of the hierarchy patterns. Furthermore, we observe that the gap between ComplEx-N3 and InjEx-C grows larger. This is against the intuition that as more triples are provided, prior knowledge may become less helpful since more information can be obtained from triples. This result demonstrates that prior rule knowledge can consistently provide a positive impact on the few-shot link predication tasks. 

We also observe that composition rules more positively impact FB15k-237-Zero, while hierarchy rules more positively impact NELL-One. This illustrates that different patterns have a diverse effect on supporting better few-shot relations, which motivates InjEx on injecting various types of rules. 

Further, on FB15k-237-Zero with 0 supporting triple (0-shot), the improvement between ComplEx-N3 and InjEx are statistically significant. With 0 supporting triples in training set, such prior knowledge can only be captured from the rules. The improvement indicates that InjEx is able to effectively make use of prior knowledge during prediction.

\begin{figure}
  \begin{subfigure}{\linewidth}
  \includegraphics[width=.5\linewidth]{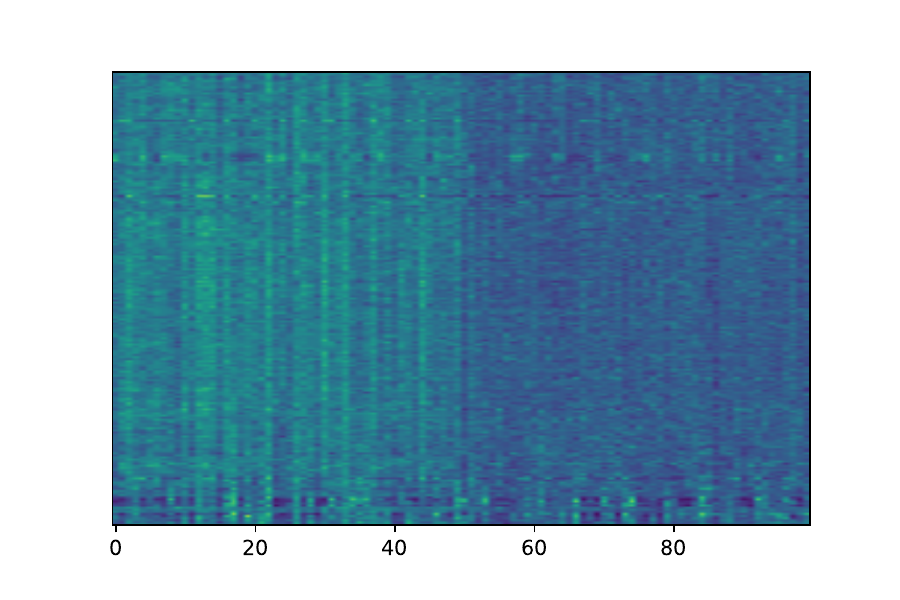}\hfill
  \includegraphics[width=.5\linewidth]{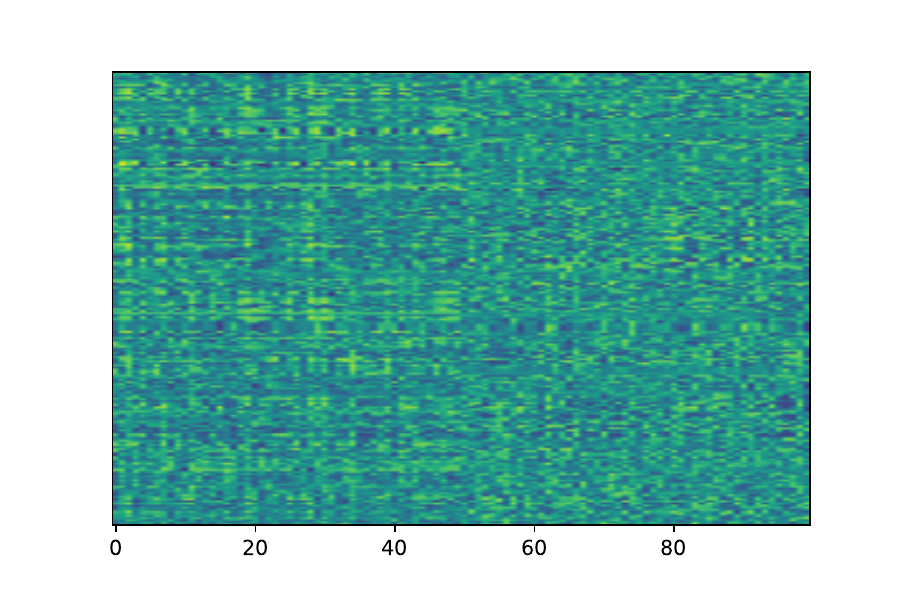}\hfill
  \caption{Entity representations}
  \label{fig:entity representation}
  \end{subfigure}\par\medskip
  \begin{subfigure}{\linewidth}
  \includegraphics[width=.5\linewidth]{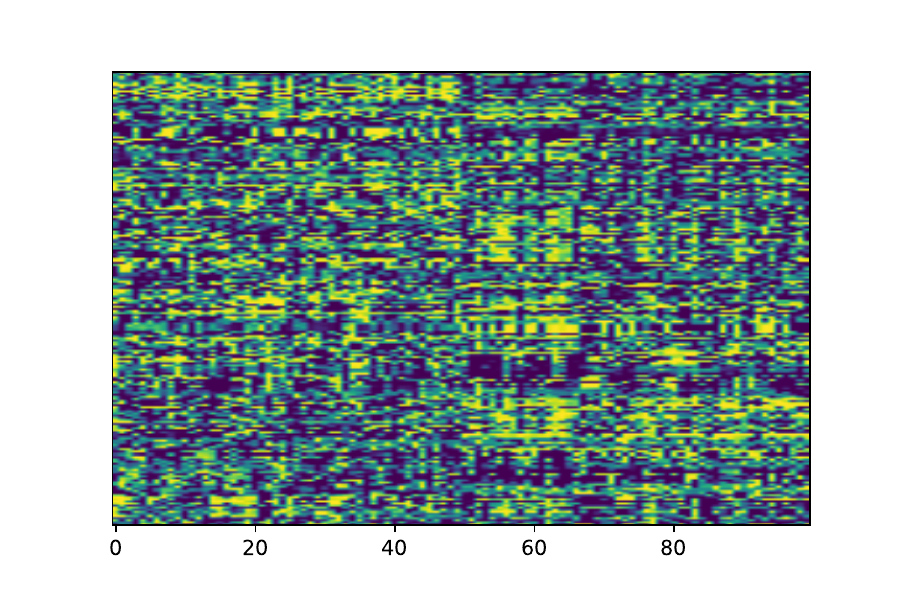}\hfill
  \includegraphics[width=.5\linewidth]{figures/interpret/ComplEx-rel.pdf}\hfill
  \caption{relation representations}
  \label{fig:relation representation}
  \end{subfigure}
  \caption{Visualization of "active" dimensions in entity and relation representations (rows) learned by InjEx (left) and ComplEx-N3 (right) on FB15k-237. Values range from 0 (purple) to 1 (yellow). Best viewed in color}
  \label{fig:embeddings-all}
\end{figure}

\begin{figure}
  \begin{subfigure}{\linewidth}
  \includegraphics[width=.5\linewidth]{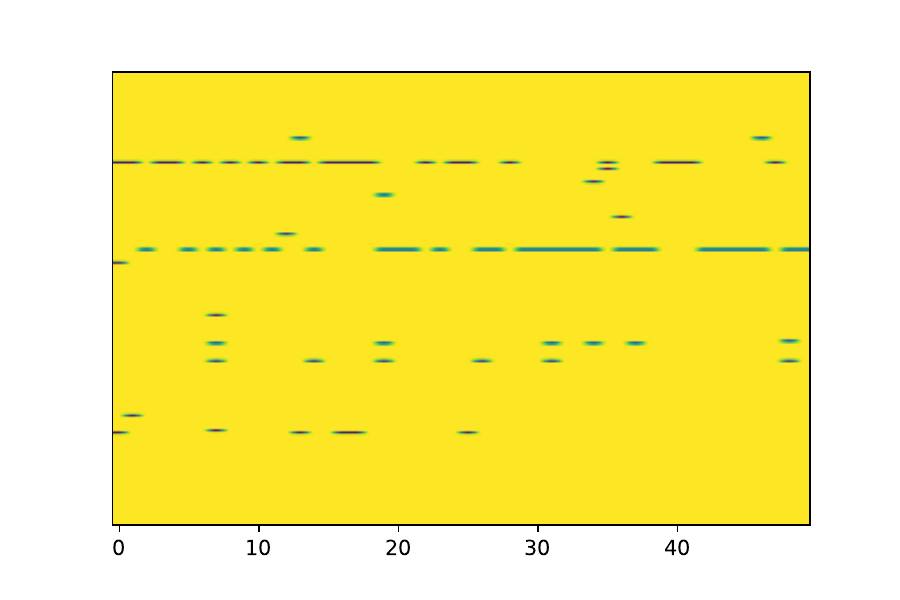}\hfill
  \includegraphics[width=.5\linewidth]{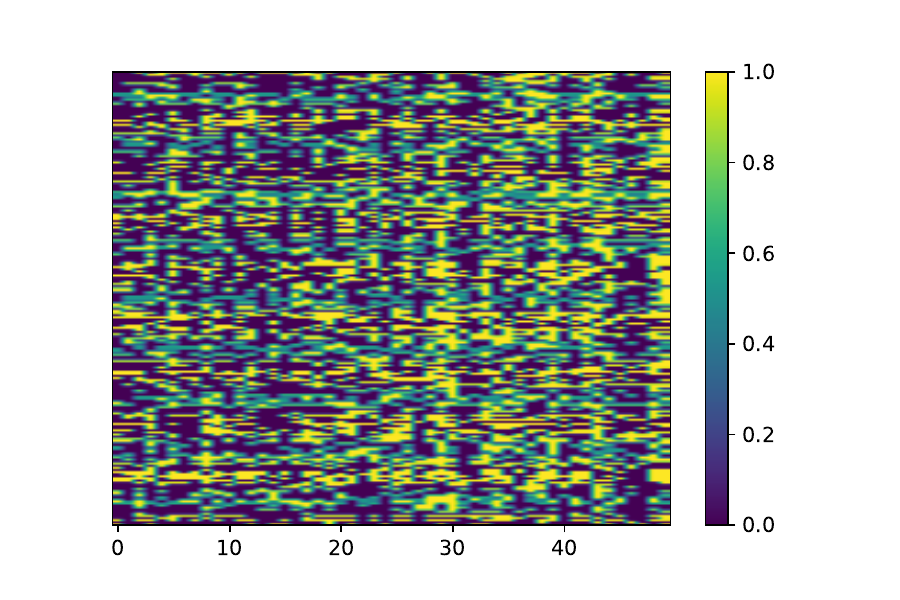}\hfill
  \caption{Difference between real components}
  \label{fig:difference of real Horn}
  \end{subfigure}\par\medskip
  \begin{subfigure}{\linewidth}
  \includegraphics[width=.5\linewidth]{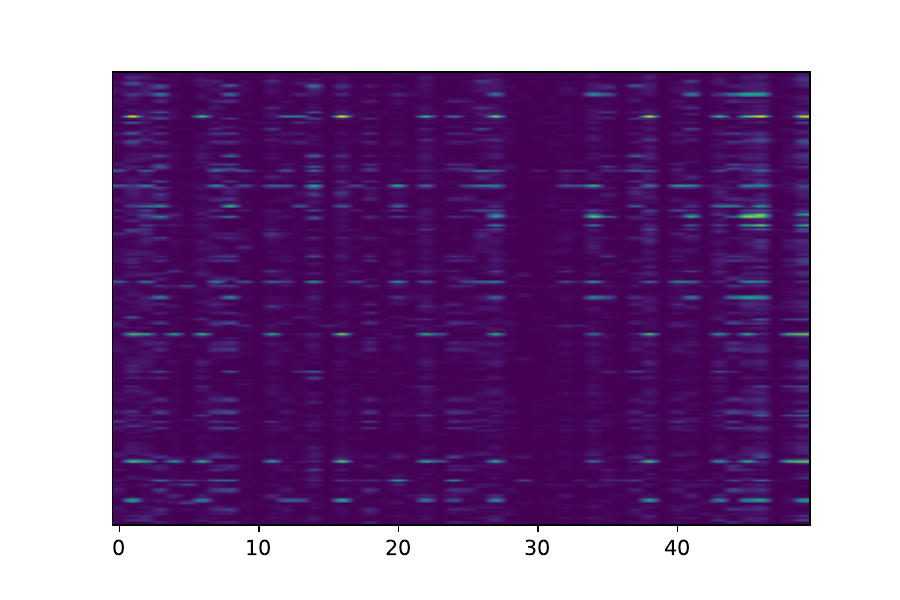}\hfill
  \includegraphics[width=.5\linewidth]{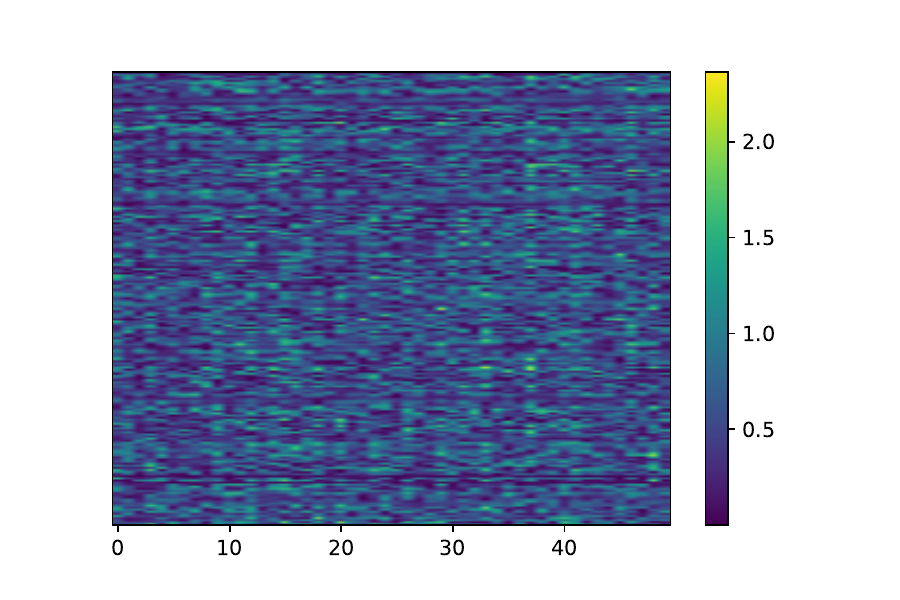}\hfill
  \caption{Difference between imaginary components}
  \label{fig:difference of imaginary Horn}
  \end{subfigure}
  \caption{Visualization of differences between real, imaginary components of head and body relation(s) representations learned by InjEx with Horn rule set (left) and ComplEx-N3 (right). Values range from low (purple) to high (yellow). Best viewed in color.}
  \label{fig:rel-embeddings-all}
\end{figure}
In this section, we demonstrate further analysis with visual inspection of the entity and relation embedding space on how InjEx improves the quality of KG embeddings when the rules are imposed. 

\subsection{Analysis on Entity Representations}
This section demonstrates how the structure of the entity embedding space changes when the constraints are imposed. We provide the visualization of all entity representations on FB15k-237. Figure\ref{fig:entity representation} shows the general distribution of all entity embeddings. We observe that the highlighted dimensions are less random comparing with ComplEx-N3. This shows that with the constraints and rules we impose, InjEx learns to focus on representative dimensions for different entities, which leads to a more efficient representation compared with ComplEx.

\subsection{Analysis on Relation Representations}
We also provide the visualization of relation representations on FB15k-237. On this dataset, each relation is associated with a single type label. To show the norm constraints and loss term we add affect the relation representations, in Figure\ref{fig:relation representation} we present the general distribution of all relation embeddings. The result indicates that InjEx obtains compact and interpretable representations for relations. Fewer dimensions are activated in InjEx compared with ComplE-N3, illustrating InjEx is more efficient in capturing representative knowledge from the background KG. 

Figure \ref{fig:rel-embeddings-all} visualizes the representations of all relations included in the Horn rule set, learned by ComplEx-N3 and InjEx. We randomly pick 50 dimensions from both their real and imaginary components. 

For real components, with Horn rules injected, we expect the representation of the head relation to be larger than or equal to (element-wise multiply of) the body relation(s). In  Figure\ref{fig:difference of real Horn}, we check whether the previous statement holds for all the rules. The result shows that with InjEx, the representations of relations do follow the guidance of in Eq\ref{eq:composition rules} and Eq\ref{eq:Horn rules} while representations learned by ComplEx-N3 are rather arbitrary. 

For imaginary components, Horn rules tend to force head relation(s) to be similar to the (element-wise multiply of) body relation(s). Figure \ref{fig:difference of imaginary Horn} shows that the constraints and loss term in InjEx guide the learning of relation representations, shrinking the gap between head and body relation(s) compared with ComplEx-N3. 

\begin{figure}
  \begin{subfigure}{\linewidth}
  \includegraphics[width=\linewidth]{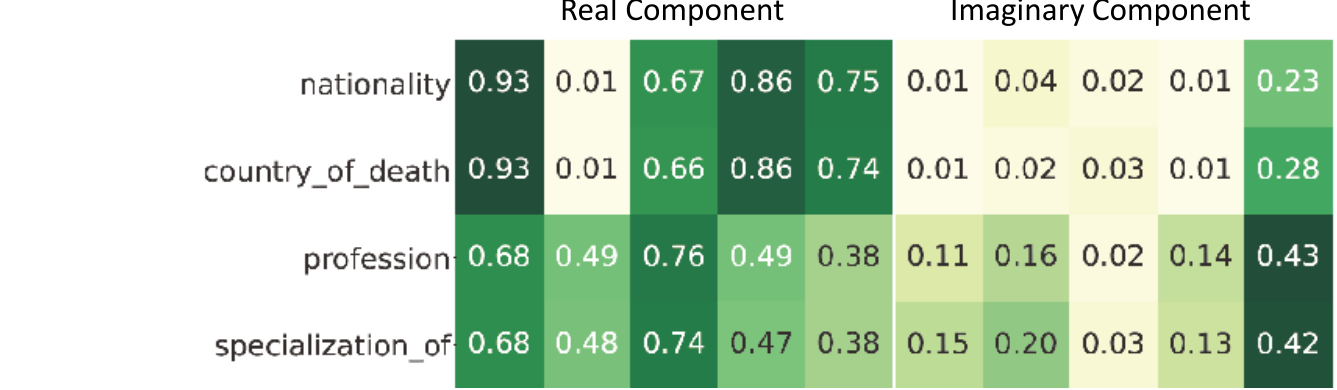}\hfill
  \caption{Examples of relation representations of hierarchy rules}
  \label{fig:hierarchy examples}
  \end{subfigure}\par\medskip
  \begin{subfigure}{\linewidth}
  \includegraphics[width=\linewidth]{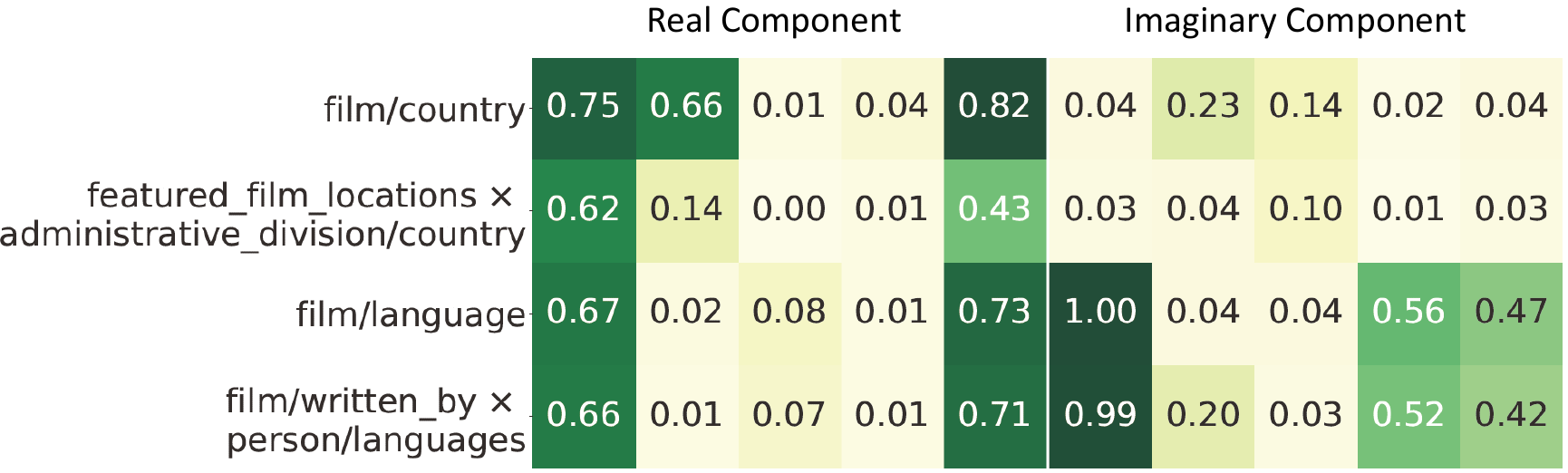}\hfill
  \caption{Examples of relation representations of composition rules}
  \label{fig:composition examples}
  \end{subfigure}
  \caption{Relation representations learned by InjEx. (a) Example of hierarchy rules. (b) Example of composition rules}
  \label{fig:rule-examples-all}
\end{figure}

To further show how InjEx affects the learning of relation representations, We visualize the representations of two pairs of relations from hierarchy rules and two pairs of relations from composition rules learned by InjEx in Figure \ref{fig:hierarchy examples} and Figure\ref{fig:composition examples}. For each relation, we randomly pick 5 dimensions from both its real and imaginary components. We present:
\begin{align*}
    /people/person/&nationality \Leftarrow \nonumber \\ 
    & /people/deceased\_person/place\_of\_death \nonumber \\
     /people/person/&profession \Leftarrow \nonumber \\ 
    & /people/profession/specialization\_of 
\end{align*} 
\begin{align*}
    /film/film/country \Leftarrow &/film/film/featured\_film\_locations \land \nonumber \\ &/location/administrative\_division/country \nonumber \\
    /film/film/language \Leftarrow & /film/film/written\_by \land \nonumber \\ 
    & /people/person/languages 
\end{align*} 
By imposing such Horn rules, these relations can encode such logical regularities quite well. InjEx learn to force similar representations for head and body relations to follow $Re(\textbf{r}_1 \times \textbf{r}_2 \times ... \times \textbf{r}_k) / R^k \leq Re(\textbf{r}) / R$, $Im(\textbf{r}_1 \times \textbf{r}_2 \times ... \times \textbf{r}_k) / R^k \approx Im(\textbf{r}) / R$ for Horn rules, thus, impose prior knowledge into relation embeddings.

\section{Conclusion and future work} \label{sec:future-work}
In this paper we present InjEx, and proved its expressive power and ability to inject different types of rules. We empirically showed that InjEx achieves stae-of-art performance for KGC and is competitive for few-shot KGC by injecting different types of rules. InjEx can be further updated and improved by leveraging additional types of rules for both tasks. Because the results indicate that Horn rules with different length are all valuable for KGC, it is worth exploring further interactions among other rule combinations in order to maximize the possible impact of various types of rules. 

\bibliographystyle{ACM-Reference-Format}
\bibliography{ref}

\appendix
\section{Proof of THEOREM~\ref{thm:composition}} \label{prf: constraint proof}
\begin{proof} 

With the restrictions: $|r_i| \in (0,R]$ and $Re(e_i), Im(e_i) \in (0, 1]$, for every dimension $l$ we have
\begin{align*}
    \varphi_{1l}&:=\varphi_l(e_1, r_1, e_2)\\
    &=|(|r_{1l}||e_{1l}||e_{2l}|cos(\theta_{r_{1l}} + \theta_{e_{1l}} - \theta_{e_{2l}}))| \in [0, 2R]\\
    \varphi_{2l}&:=\varphi_l(e_2, r_2, e_3)\\
    &=|(|r_{2l}||e_{2l}||e_{3l}|cos(\theta_{r_{2l}} + \theta_{e_{2l}} - \theta_{e_{3l}}))| \in [0, 2R]
\end{align*}
which means we can treat $\Pi_l(|\varphi_{il}/(2R)|)$ like the probability of triple $i$ being True. Naturally, for rule $r_1, r_2 \rightarrow r_3$ to hold, it is sufficient if we prove that for each dimension $l$,
\begin{align}
\begin{split}
\label{prob-goal}
    |\varphi_{3l} / (2R)| &:= |(|r_{3l}||e_{1l}||e_{3l}|cos(\theta_{r_{3l}} + \theta_{e_{1l}} - \theta_{e_{3l}}))/(2R)|\\
                    &\geq |\varphi_{1l}/(2R)| * |\varphi_{2l}/(2R)|
\end{split}
\end{align}

We will prove the above (\ref{prob-goal}) in two steps (the proof is element-wise, for simplicity we drop footnote $l$ in all notations). Let $r := \langle r_1, r_2 \rangle / R $, we will show:
\begin{align}
    \label{prf-step1}
    &|\varphi_1/(2R)| * |\varphi_2/(2R)| \stackrel{?}{\le} |\varphi_r/(2R)|\\
    \label{prf-step2}
    &|\varphi_r/(2R)| \stackrel{?}{\le} |\varphi_3/(2R)|
\end{align}

Let $\alpha := \theta_{r_1} + \theta_{e_1} - \theta_{e_2}$, $\beta := \theta_{r_2} + \theta_{e_2} - \theta_{e_3}$, and for $r = \langle r_1, r_2 \rangle/R$, we have $\theta_r$ = $\theta_{r_1}$ + $\theta_{r_2}$.

First we prove (\ref{prf-step1}). When triples $(e_1, r_1, e_2)$ and $(e_2, r_2, e_3)$ are both True, $\varphi_1$ and $\varphi_2$ reach maximum value, which means $\alpha \approx \beta \approx 0$. 

Notice that for $\varphi_r$, $\theta_r + \theta_{e_1} - \theta_{e_3} = \theta_{r_1} + \theta_{r_2} + \theta_{e_1} - \theta_{e_3} = \alpha + \beta \approx 0$. Then~\ref{prf-step1} becomes $|e_1||e_2|^2|e_3| / 2 \leq |e_1||e_3|$, which holds given that $|e_i|^2 \leq 2$.

To prove (\ref{prf-step2}) we only need to show $|\varphi_r| \le |\varphi_3|$, which according to \cite{ding2018improving} is valid as long as they are positive.
\end{proof}


\end{document}